\documentclass[conference]{IEEEtran}
\IEEEoverridecommandlockouts
 \usepackage{microtype}
\usepackage{graphicx}
\usepackage{subfigure}
\usepackage{booktabs} 
\usepackage{multirow}
\usepackage{hyperref}
\usepackage{algorithm}
\usepackage{algorithmic}

\usepackage{amsmath}
 \usepackage{amssymb, bm}
\usepackage{mathtools}
\usepackage{amsthm}

\def\xm{\mathbf{X}}
\def\xv{\mathbf{x}}
\def\yv{\mathbf{y}}
\def\vv{\mathbf{v}}
\def\pv{\mathbf{p}}
\def\wv{\mathbf{w}}
\def\bv{\bm{\beta}}

\DeclareMathOperator*{\argmin}{argmin}

\theoremstyle{plain}
\newtheorem{theorem}{Theorem}[section]

\theoremstyle{Definition}

\theoremstyle{remark}

\newcommand{\diag}[1]{\operatorname{diag}\left(#1\right)}
\newcommand{\sign}[1]{\operatorname{sign}\left(#1\right)}
\AtBeginDocument{%
  \providecommand\BibTeX{{%
    \normalfont B\kern-0.5em{\scshape i\kern-0.25em b}\kern-0.8em\TeX}}}

\begin{document}

\title{On Regularized Sparse Logistic Regression}

\author{
	\IEEEauthorblockN{\ Mengyuan Zhang}
	\IEEEauthorblockA{
		\textit{Clemson University}\\
	\ 	Clemson, SC, USA\\
	\	\ mengyuz@clemson.edu
	}\and
	\IEEEauthorblockN{\ Kai Liu}
	\IEEEauthorblockA{ 
		\textit{Clemson University}\\
		\ \ Clemson, SC, USA\\
		\ \ \ kail@clemson.edu
	}
}
\maketitle 

\begin{abstract}
 Sparse logistic regression is for classification and feature selection simultaneously. Although many studies have been done to solve $\ell_1$-regularized logistic regression, there is no equivalently abundant work on solving sparse logistic regression with nonconvex regularization term. In this paper, we propose a unified framework to solve $\ell_1$-regularized logistic regression, which can be naturally extended to nonconvex regularization term, as long as  certain requirement is satisfied.
 In addition, we also utilize a different line search criteria to guarantee monotone convergence for various regularization terms. Empirical experiments on binary classification tasks with real-world datasets demonstrate our proposed algorithms are capable of performing classification and feature selection effectively at a lower computational cost.
\end{abstract}

\begin{IEEEkeywords}
logistic regression, sparsity, feature selection
\end{IEEEkeywords}

\section{Introduction}
Logistic regression has been applied widely in many areas as a method of classification. 
The goal of logistic regression is to maximize the likelihood  based on the observation of training samples, with its objective function formulated as follows with a natural and meaningful probabilistic interpretation:
\begin{equation}
\begin{split}
    \min_{\bv} \sum_{i=1}^{n} -\ln {p(y_i | \xv_i; \bv)}
    = \sum_{i=1}^{n} \ln (1+exp(\bv^T\xv_i)) - y_i \bv^T\xv_i,
\end{split}
\end{equation}
where $\xv_i$  and $y_i$ denote the $i$-th sample and its label. 

Though logistic regression is straightforward and effective, its performance can be diminished due to over-fitting~\cite{hawkins2004problem}, especially when the dimensionality of features is very high compared to the number of available training samples. Therefore, regularization term is usually introduced to alleviate over-fitting issue~\cite{ng2004feature}. 
Also, in applications with high-dimensional data, it's desirable to obtain sparse solutions, since in this way we are conducting classification and feature selection at the same time. Therefore, $\ell_1$-regularized logistic regression has received more attention with the sparsity-inducing property and its superior empirical performance~\cite{abramovich2018high}.

More recent studies show that $\ell_1$-norm regularization may suffer from implicit bias problem that would cause significantly biased estimates, and such bias problem can be mitigated by a nonconvex penalty~\cite{hastie2015statistical}. Therefore, nonconvex regularization has also been studied to induce sparsity in logistic regression~\cite{yuan2023feature}. 
Solving sparse logistic regression with convex $\ell_1$-norm regularizer or with nonconvex term using a unified algorithm has been studied in~\cite{loh2013regularized}. However, it imposes strong regularity condition on the nonconvex  term and it transfers the nonconvex problem into an $\ell_1$-norm regularized surrogate convex function, which limits its generality. As a contribution, we extend the scope of nonconvex penalties to a much weaker assumption and compare the performance of different regularization terms with a unified optimization framework.\par
In this paper, we solve $\ell_1$-regularized (sparse) logistic regression by proposing a novel framework, which can be applied to non-convex regularization term as well.
The idea of our proposed method  stems from the well know Iterative Shrinkage Thresholding Algorithm (ISTA) and its accelerated version Fast Iterative Shrinkage Thresholding Algorithm (FISTA)~\cite{beck2009fast}, upon which we modify the step-size setting and line search criteria to make the algorithm applicable for both convex and nonconvex regularization terms with empirical faster convergence rate. 
To be clear, we call any logistic regression with regularization term that can produce sparse solutions as \textit{sparse logistic regression}, therefore, the term is not only limited to $\ell_1$-norm regularization.  
\section{Related Work} \label{sec:2}




Due to the NP-hardness of $\ell_0$-norm constraint problem, being its tightest convex envelope, $\ell_1$-norm is widely taken as an alternative to induce sparsity~\cite{liu2018high,liu2018learning,zhang2022enriched}. 
The main drawback of $\ell_1$-norm is it's non-differentiable, which makes the computation challenging compared to squared $\ell_2$-norm. Sub-gradient is an option but can be very slow even if we disregard the fact that it is non-monotone. Besides, there has been active research on numerical algorithms to solve $\ell_1$-regularized logistic regression. 
Among these, an intuitive idea is as the challenge originates from the non-smoothness of $\ell_1$-norm, we can make it \lq{smoothable}\rq. 
For example, in~\cite{schmidt2007fast}, the $\ell_1$-norm is approximated by a smooth function that is readily solvable by any applicable optimization method. The \textit{iteratively reweighted least squares least angle regression} (IRLS-LARS) method converts the original problem into a smooth  function by the equivalent $\ell_1$-norm constrained ball~\cite{lee2006efficient}. 
In addition, coordinate descent method has been utilized to solve the $\ell_1$-regularized logistic regression as well: in ~\cite{genkin2007large} and ~\cite{madigan2005bayesian}, cyclic coordinate descent is used for optimizing Bayesian logistic regression. Besides, the interior-point method is another option with truncated Newton step and conjugated gradient iterations~\cite{koh2007interior}. We refer readers to the papers and references therein.



The bias of $\ell_1$-norm is implicitly introduced as it penalizes the parameters with larger coefficients more than the smaller ones. Recently, nonconvex regularization terms have drawn considerable interest in sparse logistic regression since it is able to ameliorate the bias problem of $\ell_1$-norm, and it acts as one main driver of recent progress in nonconvex and nonsmooth optimization~\cite{wen2018survey}. 


It is worth noting that there have been many studies trying to solve the $\ell_0$-regularized problems directly. 
The most common method is to conduct coordinate descent for nonconvex regularized logistic regression~\cite{breheny2011coordinate,mazumder2011sparsenet}.
The alternating direction method of multipliers (ADMM) inspires the development of incremental aggregated proximal ADMM to solve nonconvex optimization problems, and it achieves good results in sparse logistic regression~\cite{jia2021incremental}. 
In~\cite{marjanovic2015mist}, the momentumized iterative shrinkage thresholding (MIST) algorithm is proposed to minimize the nonconvex criterion for linear regression problems, and  similar ideas can be applied to logistic regression as well. Besides the $\ell_0$-norm, other nonconvex penalties are also explored. The minimax concave penalty (MCP) is studied in~\cite{zhang2010nearly} together with a penalized linear unbiased selection (PLUS) algorithm, and it shows the MCP is unbiased with superior selection accuracy.
The smoothly clipped absolute deviation (SCAD) penalty is proposed in~\cite{fan2001variable}, which corresponds to a quadratic spline function, and the study shows  SCAD penalty outperforms the $\ell_1$-norm regularizer significantly, and it has the best performance in selecting significant variables without introducing excessive biases. 
\begin{table}
	\begin{center}
		\caption{Nonconvex regularization terms in sparse logistic regression.}
		\label{tab:nonconvexPenalty}
		\begin{tabular}{ c c }
			\toprule
			Name & Formulation \\
			\midrule
			$\ell_p$-norm, $0 < p < 1$ & $\lambda |\beta_i|^q$ \\  
			\midrule
			Capped $\ell_1$-norm & $\lambda\min{( |\beta_i|, \epsilon)}$ \\
			\midrule
			SCAD & $\begin{cases} 
				\lambda|\beta_i|, &  |\beta_i| \leq \lambda \\
				\frac{-\beta_i^2 + 2\theta\lambda|\beta_i|-\lambda^2}{2(\theta-1)}, &  \lambda < |\beta_i| \leq \theta\lambda \\
				\frac{(\theta+1)\lambda^2}{2}, & | \beta_i| > \theta\lambda \\
			\end{cases}$\\
			\midrule 
			MCP & $\begin{cases}
				\lambda|\beta_i|-\frac{\beta_i^2}{2\theta}, & \text{if} |\beta_i| \leq \theta\lambda \\
				\frac{\theta\lambda^2}{2}, & \text{if} |\beta_i| > \theta\lambda\\
			\end{cases}$ \\
			\bottomrule
		\end{tabular}
	\end{center}
\end{table}

\section{Optimization Algorithms} \label{sec:3}

%
\subsection{Algorithms for $\ell_1$ Regularized Sparse Logistic Regression}

We first consider  sparse logistic regression problem as:
\begin{equation}\begin{split}
\label{ol1}
      \min_{\bv} f(\bv)= \underbrace{\sum_i ln(1+\exp(\xv_i^T\bv))-y_i(\xv_i^T\bv)}_{l(\bv)}+\underbrace{\lambda \|\bv\|_1}_{g(\bv)}.
\end{split}\end{equation}




\begin{theorem}\label{th:lip}
$\nabla l(\bv)$ in Eq~(\ref{ol1}) is Lipschitz continuous, and it's Lipschitz constant is $L:= \frac{1}{4}\lambda_{max}(\xm\xm^T)$.
\end{theorem}

\begin{proof}
	Let $\xm$ be the matrix of training samples, where the $i$-th column $\xv_i$ represents the $i$-th sample.
	We have
	\begin{equation}
		\nabla l(\bv) = \xm(\pv-\yv), \textrm{where } p_i = \frac{1}{1+exp(-\langle \bv, \xv_i \rangle)}
	\end{equation}
	and
	\begin{equation}
		\nabla^2 l(\bv) = \xm \diag{(p_i(1-p_i))}\xm^T.
	\end{equation}
	By the mean value theorem, we know there exists
	$\mathbf{c} \in (\bv, \bv+\Delta)$
	such that 
		$\nabla l(\bv+\Delta) - \nabla l(\bv) = \nabla^2 l(\mathbf{c}) \Delta$.
	Since
	\begin{equation}
		\begin{split}
			\|\nabla^2 l(\bv)\| & \leq \lambda_{max}(\xm \diag{(p_i(1-p_i))}\xm^T) \\
			& = \textrm{max}_{\|\vv\|=1} \vv^T \xm \diag{p_i(1-p_i)}\xm^T \vv \\
			& = \textrm{max}_{\|\vv\|=1} \vv^T [\sum (\xv_i p_i(1-p_i) \xv_i^T)] \vv \\
			& \leq \frac{1}{4} \textrm{max}_{\|\vv\|=1} \vv^T [\sum (\xv_i \xv_i^T)] \vv 
			= \frac{1}{4}\lambda_{max}(\xm\xm^T),
		\end{split}
	\end{equation}
	we have
	\begin{equation}
		\begin{split}
			& \|\nabla l(\bv+\Delta) - \nabla l(\bv)\| =
			\|\nabla^2 l(\mathbf{c}) \Delta\| \\ \leq &\|\nabla^2 l(\mathbf{c})\|\|\Delta\|
			\leq  \frac{1}{4}\lambda_{max}(\xm\xm^T) \|\Delta\|,
		\end{split}
	\end{equation}
	thus $\nabla l(\bv)$ is Lipschitz continuous with $\frac{1}{4}\lambda_{max}(\xm\xm^T)$.
\end{proof}

We use ISTA and  FISTA with backtracking line search to solve Eq~(\ref{ol1}), which are described in Algorithm~\ref{alg:l1ista} and ~\ref{alg:l1fista} respectively, where $p_L(\bv)$ represents the proximal operator defined as
    $p_L(\bv) = \argmin_{\wv} \frac{L}{2} \|\wv - (\bv - \frac{1}{L} \nabla l(\bv) \|^2 + g (\wv)$.
The line search stopping criterion is:
$f(p_{\bar{L}}(\bv_{k-1})) \leq q_{\bar{L}}(p_{\bar{L}}(\bv_{k-1}),\bv_{k-1})$,
where $f(\bv)$ is defined in Eq~(\ref{ol1}), and 
\begin{equation}\begin{split}
& q_L(p_L(\bv), \bv) \\
=&  l(\bv) + \langle p_L(\bv)-\bv, \nabla l(\bv) \rangle +\frac{L}{2}\| p_L(\bv) - \bv\|^2 + g(p_L(\bv)).
\end{split}\end{equation}

One potential drawback of the vanilla ISTA and FISTA is the initial step size which is randomly set as $L_0>0$ and keep increasing $L_k$ during update to satisfy the line search stopping criterion. In case $L_0$ is larger than the Lipschitz constant $L$ of $\nabla l(\bv)$, the step size can be too small to obtain optimal solution rapidly~\cite{zhang2023multi}. Thus, different from vanilla ISTA, in Algorithm~\ref{alg:l1ista}, we first initialize the stepsize with Lipschitz continuous constant and then utilize Barzilai-Borwein (BB) rule to serve as a starting point for backtracking line search:
\begin{equation}\begin{split}\label{eq:bb}
	\bm{\delta}_k = \bv_{k-1} - \bv_{k-2},
  \mathbf{v}_k = \nabla l(\bv_{k-1})-\nabla l(\bv_{k-2}), 
  L_k = \frac{\langle \bm{\delta}_k, \mathbf{v}_k \rangle}{\langle \bm{\delta}_k, \bm{\delta}_k \rangle}.
\end{split}\end{equation}
Experiments on synthetic data show that with BB rule, ISTA with randomly initialized step size will admit faster convergence during update, which is demonstrated in Figure~\ref{fig:randomBB}. 
\begin{figure}[!h]
\begin{minipage}[t]{0.485\linewidth}
    \centering
    \includegraphics[width=0.95\textwidth]{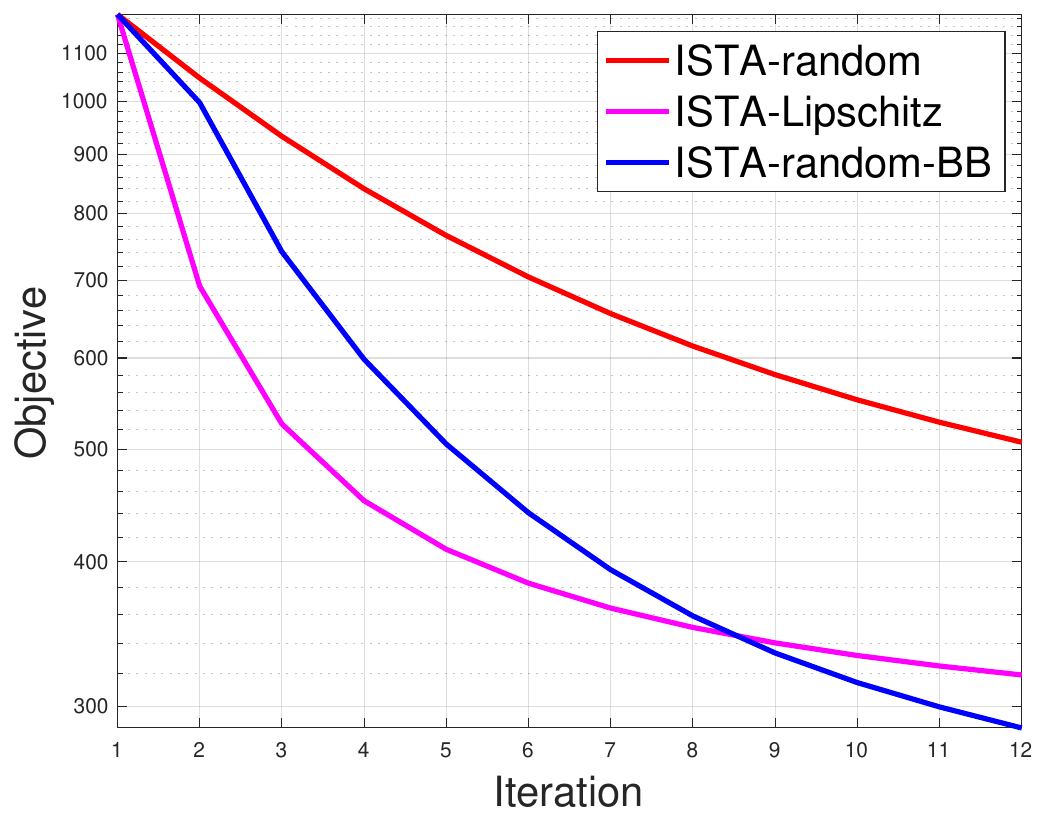}
    \caption{ISTA with different settings.}
    \label{fig:randomBB}
\end{minipage}
\hspace{0.1cm}
\begin{minipage}[t]{0.485\linewidth} 
    \centering
    \includegraphics[width=0.95\textwidth]{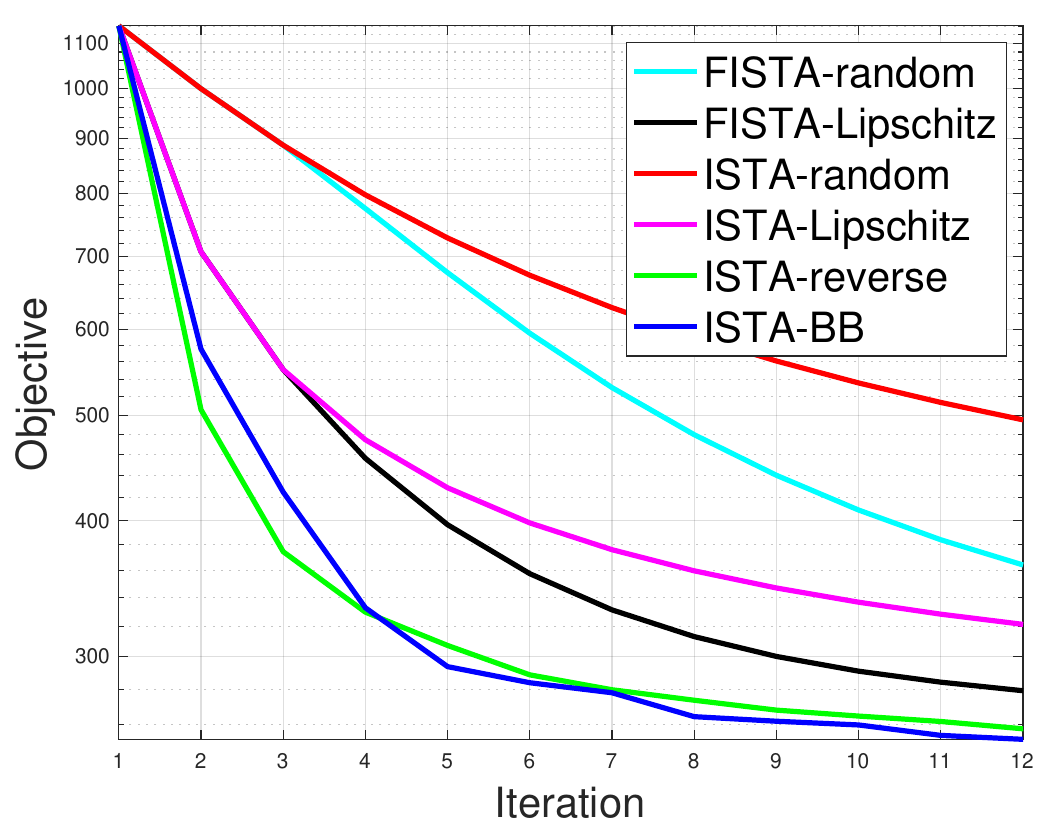}
    \caption{$\mathcal{L}_1$-norm: convergence plots.}
    \label{fig:convexAll}
\end{minipage}        
\end{figure}  
In Algorithm~\ref{alg:l1fista}, we also initialize the step size as $\frac{1}{L}$ to avoid too small step size. However, to guarantee $O(\frac{1}{k^2})$ convergence rate for FISTA, we still need the step size to be monotonically nonincreasing, thus BB rule cannot be utilized in FISTA. 
\begin{algorithm}[h!]
\caption{ISTA-BB: ISTA with Lipschitz constant and BB rule to solve Eq~(\ref{ol1}).}
\label{alg:l1ista}
\begin{algorithmic}
	\STATE Initialize ${\bv_0}$, step size $\frac{1}{L_0}$ as $\frac{1}{L}$, where $L$ is the Lipschitz constant of $\nabla l(\bv)$, set $\eta>1$;
	\REPEAT
        \STATE 1) Start from $k=2$, update the step size $\frac{1}{L_k}$ using the Barzilai-Borwein (BB) rule
        \STATE 2) Find the smallest nonnegative integer $i_k$ such that with $\bar{L} = \eta^{i_k}L_{k}$ 
        we have
        $f(p_{\bar{L}}(\bv_{k-1})) \leq q_{\bar{L}}(p_{\bar{L}}(\bv_{k-1}),\bv_{k-1})$\;.
	\STATE 3) Set $L_k = \eta^{i_k}L_{k}$ and update $\bv_k = p_{L_k}(\bv_{k-1})$\;
\UNTIL{convergence}
	\end{algorithmic}
\end{algorithm}

\begin{algorithm}[h!]
\caption{FISTA-Lipschitz: FISTA with Lipschitz constant to solve Eq~(\ref{ol1}).}
\label{alg:l1fista}
\begin{algorithmic}
	\STATE Initialize ${\bv_0}$, step size $\frac{1}{L_0}$ as $\frac{1}{L}$, where $L$ is the Lipschitz constant of $\nabla l(\bv)$, set $\eta>1$, $\wv_1 = \bv_0$, $t_1 = 1$;
	\REPEAT  
        \STATE 1) Find the smallest nonnegative integer $i_k$ such that with $\bar{L} = \eta^{i_k}L_{k-1}$ 
        we have
        $f(p_{\bar{L}}(\wv_{k})) \leq q_{\bar{L}}(p_{\bar{L}}(\wv_{k}),\wv_k)$\;.
	\STATE 2) Set $L_k = \eta^{i_k}L_{k-1}$ and update 
 \begin{equation}\begin{split}
&  \bv_k = p_{L_k}(\wv_{k}) \\
& t_{k+1} = \frac{1+\sqrt{1+4t_k^2}}{2} \\
& \wv_{k+1} = \bv_k + \frac{t_k - 1}{t_{k+1}}(\bv_k - \bv_{k-1})
 \end{split}\end{equation}
\UNTIL{convergence}
	\end{algorithmic}
\end{algorithm}

Besides the step size setting in BB rule aforementioned, another option proposed by us is to find the largest step-size by searching \textit{reversely}: we start by setting the step-size to $\frac{1}{L}$ in each iteration and keep enlarging it until the line search condition is not satisfied and take the last step-size satisfying the criterion. In this way, we are able to find the largest step size satisfying the line search criterion for each update iteration. 
The proposed method is summarized in Algorithm~\ref{alg:l1reverse}. Again, unlike the conventional ISTA, the step size in Algorithm~\ref{alg:l1reverse} is not decreasing monotonically, and to compute the largest singular value with lower computational cost, we can utilize power iteration method~\cite{booth2006power} to get rid of the computationally expensive singular value decomposition. 
\begin{algorithm}[h!]
\caption{ISTA-reverse: ISTA with Lipschitz constant and reverse step size searching to solve Eq~(\ref{ol1}).}
\label{alg:l1reverse}
\begin{algorithmic}
	\STATE Initialize ${\bv_0}$ randomly, step size $\frac{1}{L_0}$ as $\frac{1}{L}$, where $L$ is the Lipschitz constant of $\nabla l(\bv)$, set $\eta>1$;
	\REPEAT
        \STATE 1) Find the smallest nonnegative integer $i_k$ such that with $\bar{L} = L_0/ \eta^{i_k}$ 
        we have
        $f(p_{\bar{L}}(\bv_{k-1})) > q_{\bar{L}}(p_{\bar{L}}(\bv_{k-1}),\bv_{k-1})$\;.
	\STATE 2) Set $L_k = L_0/ \eta^{i_k-1}$ and update $\bv_k = p_{L_k}(\bv_{k-1})$\;
\UNTIL{convergence}
	\end{algorithmic}
\end{algorithm}

By Algorithm~\ref{alg:l1reverse}, the objective decreases much faster than  vanilla ISTA. From Figure~\ref{fig:convexAll} we can see that for FISTA, when we initialize the step-size with $\frac{1}{L}$ (FISTA-Lipschitz), it has better convergence performance than initialized with a random number (FISTA-random), which might be smaller than $\frac{1}{L}$. For ISTA, the two algorithms proposed by us (ISTA-BB and ISTA-reverse) have similar performance and they obviously outperform vanilla ISTA with backtracking line search, either ISTA-random or ISTA-Lipschitz. \\
The convergence proof of Algorithm~\ref{alg:l1ista}, Algorithm~\ref{alg:l1fista}, and Algorithm~\ref{alg:l1reverse} can be easily adapted from the proof in~\cite{beck2009fast}, here we only present the key theorems of the convergence rate. 


\begin{theorem}\label{them:ista}
Let $\{\bv_k\}$ be the sequence generated by either Algorithm~\ref{alg:l1ista} or Algorithm~\ref{alg:l1reverse}, for any $k > 1$, 
\begin{equation}
    f(\bv_k) - f(\bv^*) = O(\frac{1}{k}),
\end{equation}
where $\bv^*$ is an optimal solution to Eq~(\ref{ol1}). 
\end{theorem}

\begin{theorem}\label{them:fista}
Let $\{\bv_k\}$ be the sequence generated by Algorithm~\ref{alg:l1fista}, for any $k > 1$, 
\begin{equation}
    f(\bv_k) - f(\bv^*) = O(\frac{1}{k^2}),
\end{equation}
where $\bv^*$ is an optimal solution to Eq~(\ref{ol1}). 
\end{theorem}

\subsection{Algorithms for Nonconvex Regularized Sparse Logistic Regression}

While the $\ell_1$-norm regularization is convenient since it's convex, several studies show that sometimes nonconvex regularization term can have better performance ~\cite{wen2018survey} though 
it turns the objective to nonconvex and even nonsmooth, which is challenging to obtain optimal solution.
Current literature lacks a unified yet simple framework that works for both convex  and a wide class of nonconvex regularization terms. With this consideration, we would like to list a bunch of nonconvex regularization terms that on one hand can ameliorate the bias problem, and on the other hand, can be solved with the same algorithms for the convex term.\\
\begin{table*}[htp!]
	\begin{center}
		\caption{Examples of regularization terms and corresponding proximal operators.}
		\label{tab:nonconvex}
		\begin{tabular}{c|c|c|c}
			\toprule
			\multicolumn{1}{c}{Penalty} &
			\multicolumn{1}{c}{$g_1(\beta_i)$} &
			\multicolumn{1}{c}{$g_2(\beta_i)$}  &
			\multicolumn{1}{c}{Proximal operator}  \\
			\midrule
			SCAD 
			& $\lambda|\beta_i|$ & 
			$\begin{cases} 
				0, &  |\beta_i| \leq \lambda \\
				\frac{\beta_i^2 - 2\lambda|\beta_i|+\lambda^2}{2(\theta-1)}, &  \lambda < |\beta_i| \leq \theta\lambda \\
				\lambda|\beta_i| - \frac{(\theta+1)\lambda^2}{2}, &  | \beta_i| > \theta\lambda \\
			\end{cases}$
			
			& $\begin{cases}
				\sign{t} \max{(|t|-\lambda), 0}, & |t| \leq 2\lambda \\
				\frac{(\theta-1)t - \sign{t} \theta\lambda}{\theta-2}, & 2\lambda < |t| \leq \theta\lambda \\
				t,  & |t| > \theta\lambda 
			\end{cases}$
			
			\\
			\midrule
			MCP 
			& $\lambda|\beta_i|$ & 
			$\begin{cases}
				\frac{\beta_i^2}{2\theta}, & \text{if} |\beta_i| \leq \theta\lambda \\
				\lambda|\beta_i| - \frac{\theta\lambda^2}{2}, & \text{if} |\beta_i| > \theta\lambda\\
			\end{cases}$  
			&
			$\begin{cases}
				0, & |t| \leq \lambda \\
				\frac{\sign{t} (|t|-\lambda)}{1-1/\theta}, & \lambda < |t| \leq \theta\lambda \\
				t, & |t| > \theta\lambda\\
			\end{cases}$
			\\
			
			\midrule
			$\ell_1$-norm 
			&  &   
			&
			$\sign{t} \max{(|t|-\lambda,0)}$
			\\
			\bottomrule
		\end{tabular}
	\end{center}
\end{table*}
\noindent Nonconvex regularization terms can be written as the difference between two convex functions as long as the Hessian is bounded~\cite{yuille2001concave}. For such nonconvex penalties, we are able to solve by ISTA with slight modifications.
\begin{equation}\begin{split}
		\label{nonconvex}
		\min_{\bv} f(\bv)= \underbrace{\sum_i ln(1+\exp(\xv_i^T\bv))-y_i(\xv_i^T\bv)}_{l(\bv)}+\underbrace{g_1(\bv)-g_2(\bv)}_{g(\bv)},
\end{split}\end{equation}
where the Hessian of $g(\bv)$ is bounded. We summarize our methods in Algorithm~\ref{alg:nonconvex1} and ~\ref{alg:nonconvex2} with modified backtracking line search criteria:
\begin{equation}
    f(\bv_{k+1}) \leq f(\bv_k) - \frac{L_k}{2}\|\bv_{k+1} - \bv_k\|^2.
\end{equation}

\begin{algorithm}[h!]
\caption{ISTA-BB: ISTA with Lipschitz constant and BB rule to solve Eq~(\ref{nonconvex}).}
\label{alg:nonconvex1}
\begin{algorithmic}
	\STATE Initialize ${\bv_0}$, step size $\frac{1}{L_0}$ as $\frac{1}{L}$, where $L$ is the Lipschitz constant of $\nabla l(\bv)$,  set $\eta>1$;
	\REPEAT
        \STATE 1) Start from $k=2$, update the step size $\frac{1}{L_k}$ using the Barzilai-Borwein (BB) rule
        \STATE 2) Find the smallest nonnegative integer $i_k$ such that with $\bar{L} = \eta^{i_k}L_{k}$ 
        we have
        $f(p_{\bar{L}}(\bv_{k-1})) \leq f(\bv_{k-1}) - \frac{\bar{L}}{2}\|p_{\bar{L}}(\bv_{k-1})-\bv_{k-1}\|^2 $\;.
	\STATE 3) Set $L_k = \eta^{i_k}L_{k}$ and update $\bv_k = p_{L_k}(\bv_{k-1})$\;
\UNTIL{convergence}
	\end{algorithmic}
\end{algorithm}

\begin{algorithm}[h!]
\caption{ISTA-reverse: ISTA with Lipschitz constant and reverse step size searching to solve Eq~(\ref{nonconvex}).}
\label{alg:nonconvex2}
\begin{algorithmic}
	\STATE Initialize ${\bv_0}$ randomly, step size $\frac{1}{L_0}$ as $\frac{1}{L}$, where $L$ is the Lipschitz constant of $\nabla l(\bv)$, set $\eta>1$;
	\REPEAT
        \STATE 1) Find the smallest nonnegative integer $i_k$ such that with $\bar{L} = L_0/ \eta^{i_k}$ 
        we have
        $f(p_{\bar{L}}(\bv_{k-1})) > f(\bv_{k-1}) - \frac{\bar{L}}{2}\|p_{\bar{L}}(\bv_{k-1})-\bv_{k-1}\|^2 $\;.
	\STATE 2) Set $L_k = L_0/ \eta^{i_k-1}$ and update $\bv_k = p_{L_k}(\bv_{k-1})$\;
\UNTIL{convergence}
	\end{algorithmic}
\end{algorithm}
The convergence of ISTA with different step-size searching methods is illustrated in Figure~\ref{fig:nonconvexAll} with SCAD.


\begin{figure}[!h]
\begin{minipage}[t]{0.485\linewidth}
    \centering
    \includegraphics[width=0.95\textwidth]{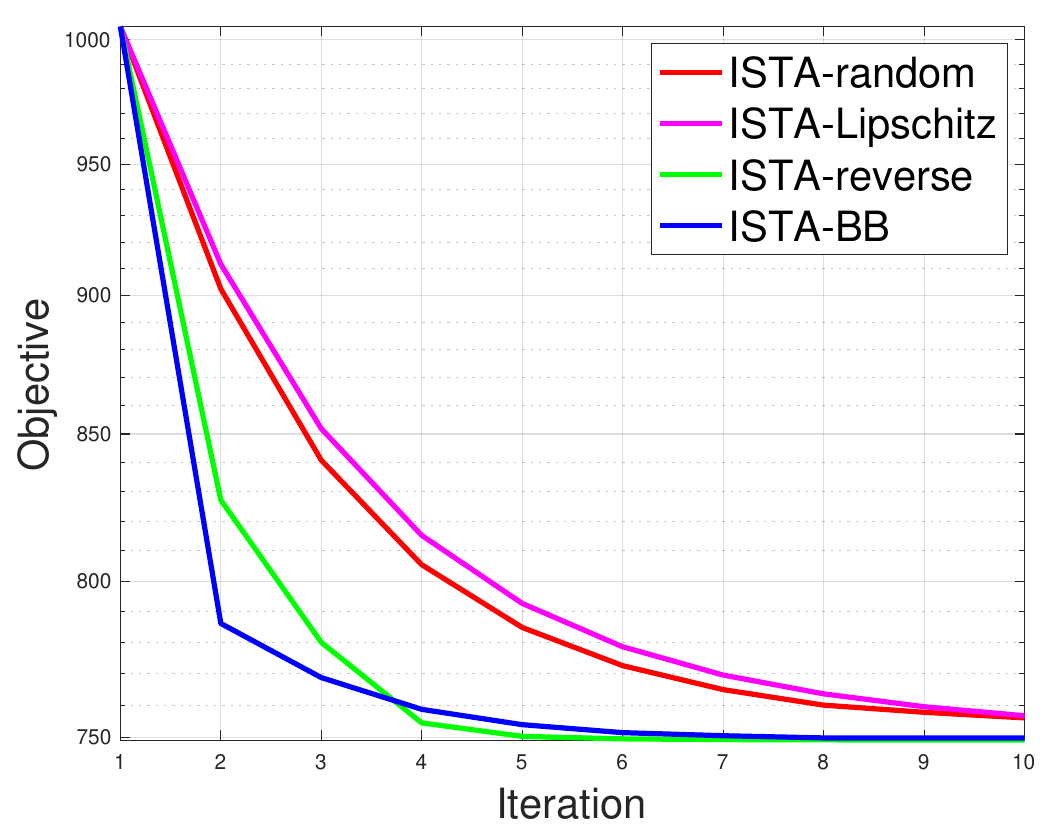}
    \caption{Nonconvex: convergence plots.}
    \label{fig:nonconvexAll}
\end{minipage}
\hspace{0.1cm}
\begin{minipage}[t]{0.485\linewidth} 
    	\centering
    \includegraphics[width=.95\linewidth]{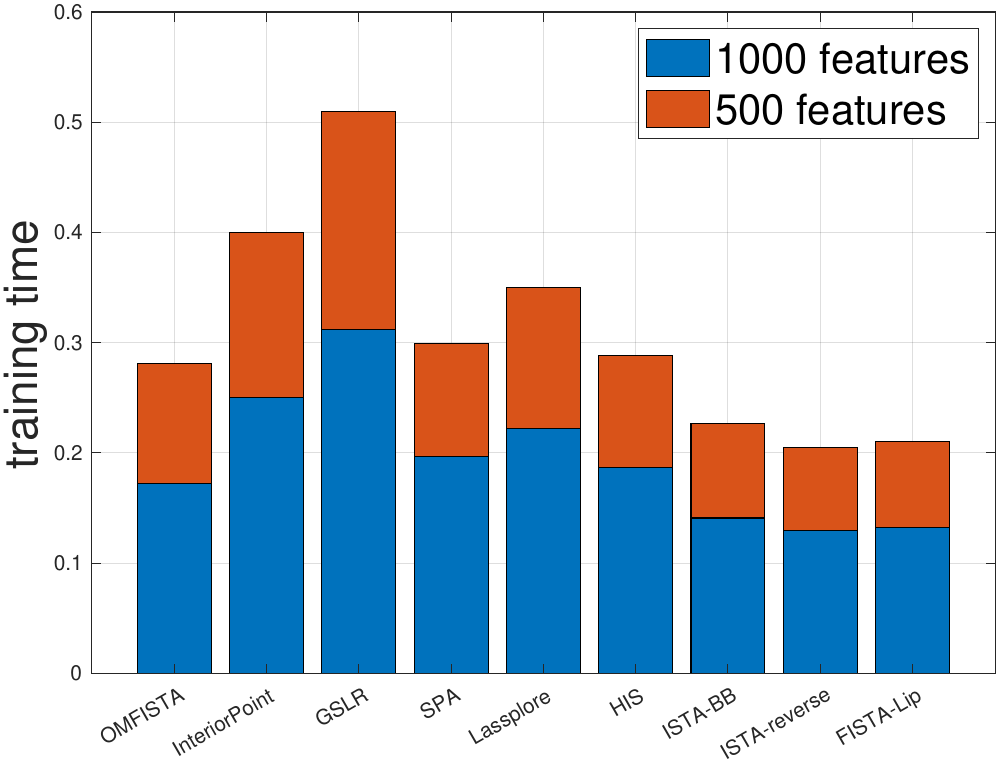}
    \caption{$\mathcal{L}_1$-norm regularized: training time comparison.}
    \label{fig:efficiency}
\end{minipage}        
\end{figure}  

\begin{theorem}\label{th:convergence}
Let $\{\bv_k\}$ be the sequence generated by either Algorithm~\ref{alg:nonconvex1} or Algorithm~\ref{alg:nonconvex2}, then all limit points of the sequence $\{\bv_k\}$ are critical points of the problem in Eq~(\ref{nonconvex}). More specifically, after $n$ iterations, we have
\begin{equation}
    \min_{0 \leq k \leq n} \|\bv_{k+1} - \bv_k\|^2 \leq \frac{2(f(\bv_0) - f(\bv^*))}{nL_{min}},
\end{equation}
where $\bv^*$ is a limit point of the sequence $\{\bv_k\}$, $L_{min}$ is the minimum $L$ among the $n$ iterations. 
\end{theorem}
\begin{proof}
	With the line search criterion, we have
	\begin{equation}
		\frac{L_k}{2} \|\bv_{k+1} - \bv_k \|^2 \leq f(\bv_k) - f(\bv_{k+1}),
	\end{equation}
	sum the above inequality we have 
	\begin{equation}
		\frac{L_{min}}{2} \sum_{k=0}^n\|\bv_{k+1} - \bv_k \|^2 \leq f(\bv_0) - f(\bv_{n+1}),
	\end{equation}
	With $f(\bv_n) \geq f(\bv^*)$, we have
	\begin{equation}
		\frac{L_{min}}{2} \sum_{k=0}^n\|\bv_{k+1} - \bv_k \|^2 \leq f(\bv_0) - f(\bv^*),
	\end{equation}
based on which we will obtain the desired conclusion.
\end{proof}

\begin{table*}[t]
\caption{Testing accuracy obtained from 5-fold cross-validation, with $\ell_1$-norm regularization.}
\begin{center}
\begin{tabular}{c | c |c c c c c c c c c c}
 \hline
 & $\lambda / \lambda_{max}$ & OMFISTA & Interior-Point & GSLR & SPA & Lassplore & HIS & Newton &  \textbf{ISTA-BB} & \textbf{ISTA-rev} & \textbf{FISTA-Lip} \\ 
\hline



 \multirow{3}{*}{Wine}  
 
 &0.02 & 0.920 &0.920 & 0.845 & 0.902 & 0.898 & 0.905 & 0.902 & 0.919 & \textbf{0.922} & \textbf{0.922} \\

 &0.1 & 0.907& 0.911 & 0.841 & 0.882 & 0.886 & 0.899 & 0.901 & 0.911 & \textbf{0.913} & 0.909\\

 &0.5 & 0.892& 0.891 & 0.825 & 0.856 & 0.852 & 0.896 & 0.895 & \textbf{0.908} & 0.902 & 0.899\\
 \hline
 \multirow{3}{*}{Specheart}  
 
 &0.02 & 0.739& 0.705 & 0.732 & 0.752 & 0.711 & 0.585 & 0.739 & 0.752 & 0.751 & \textbf{0.758} \\

 &0.1& 0.721& 0.691 & 0.716 & 0.735 & 0.702 & 0.571 & 0.722 & \textbf{0.739} & \textbf{0.739}  & 0.731\\

 &0.5 & 0.685 & 0.672 &\textbf{0.701}& 0.698 & 0.681 & 0.559 & 0.659 & 0.688 & 0.697 & \textbf{0.701}\\
 \hline
 \multirow{3}{*}{Ionosphere}  
 
 &0.02 &0.835 &  0.821 & 0.832 & 0.828 & 0.808 & 0.576 & 0.851 & 0.855 & 0.857 & \textbf{0.858} \\

 &0.1& 0.802 & 0.809 & 0.825 & 0.815 & 0.789 & 0.573 & 0.809 & 0.818 & \textbf{0.825} & 0.822\\

 &0.5 & 0.796 & 0.791 & 0.809 & 0.792 & 0.761 & 0.558 & 0.778 & 0.781 & \textbf{0.801} & \textbf{0.801}\\
 \hline
 \multirow{3}{*}{Madelon}  
 
 &0.02 & 0.615 & 0.611 & 0.612 & \textbf{0.621} & 0.601 & 0.605 & 0.615 & \textbf{0.621} & \textbf{0.621} & \textbf{0.621} \\

 &0.1& 0.603 & 0.601 & 0.601 & 0.611 & 0.592 & 0.601 & 0.611 & \textbf{0.615} & 0.611 & \textbf{0.615}\\

 &0.5 & 0.592 & 0.582 & 0.583 & 0.591 & 0.579 & 0.589 & 0.592 & 0.601 & 0.601 & \textbf{0.602}\\
 \hline
 \multirow{3}{*}{Arrhythmia}  
 
 &0.02 &0.611 & 0.591 & 0.601 & 0.608 & 0.599 & 0.589 & 0.608 & \textbf{0.615} & \textbf{0.615} & \textbf{0.615} \\

 &0.1& 0.591 & 0.588 & 0.589 & \textbf{0.595} & 0.591 & 0.578 & 0.588 & \textbf{0.595} & 0.592 & \textbf{0.595}\\

 &0.5 & 0.567 & 0.567 & 0.567 & 0.577 & 0.562 & 0.551 & 0.572 & \textbf{0.579} & \textbf{0.579} & \textbf{0.579}\\
 \hline
\end{tabular}
\end{center}
\label{tab:1}
\end{table*}

\begin{table}[t]
\caption{Testing accuracy from cross-validation with SCAD.}
\begin{center}
\begin{tabular}{c | c |c c c c}
 \hline
 & $\lambda / \lambda_{max}$ & GPGN & HONOR & \textbf{ISTA-BB} & \textbf{ISTA-rev}\\ 
\hline



 \multirow{3}{*}{Wine}  
 
 &0.02 & 0.921 & 0.925 & 0.929 & \textbf{0.931}\\

 &0.1& 0.912 & 0.912 & 0.915 & \textbf{0.917} \\

 &0.5 & 0.897 & 0.905 & 0.905 & \textbf{0.907}\\
 \hline
 \multirow{3}{*}{Specheart}  
 
 &0.02 & 0.751 & 0.761 & 0.761 & \textbf{0.763}\\

 &0.1& 0.732 & 0.735& 0.737 & \textbf{0.739}\\

 &0.5 & 0.691 & \textbf{0.711} & \textbf{0.711} & \textbf{0.711}\\
 \hline
 \multirow{3}{*}{Ionosphere}  
 
 &0.02 & 0.855 & 0.857 & \textbf{0.859} & 0.857\\

 &0.1& 0.825 & 0.827 & 0.829 & \textbf{0.831} \\

 &0.5 & 0.792 & 0.795 & 0.795 & \textbf{0.799} \\
 \hline
 \multirow{3}{*}{Madelon}  
 
 &0.02 & 0.621 & 0.625 & \textbf{0.631} & 0.628\\

 &0.1& 0.612 &0.616 & \textbf{0.619} & 0.615 \\

 &0.5 & 0.585 & 0.597 & 0.601 & \textbf{0.603}\\
 \hline
 \multirow{3}{*}{Arrhythmia}  
 
 &0.02 & 0.595 &0.609 & 0.611 & \textbf{0.618}\\

 &0.1& 0.581 & 0.592 & \textbf{0.597} & 0.595 \\

 &0.5 & 0.566 & 0.577 & 0.579 & \textbf{0.581}\\
 \hline
\end{tabular}
\end{center}
\label{tab:2}
\end{table}

\section{Experiments} \label{sec:4}



The empirical studies are conducted on the following 5 benchmark classification datasets, which can be found in  UCI machine learning repository~\cite{asuncion2007uci}:
Wine, Specheart, Ionosphere, Madelon, and Dorothea. 
For the logistic regression with convex $\ell_1$-norm penalty, we compare our proposed methods with the following counterparts: OMFISTA~\cite{zibetti2017accelerating},
Interior-Point method~\cite{koh2007interior}, GSLR~\cite{lenail2017graph}, SPA ~\cite{vono2018sparse}, Lassplore~\cite{liu2009large}, HIS~\cite{shi2010fast}, and proximal Newton~\cite{lee2014proximal}. For the logistic regression with nonconvex penalties, we compare our proposed methods with  GPGN~\cite{wang2019greedy} and HONOR ~\cite{gong2015honor}, and the nonconvex penalty we utilize in the experiment is SCAD. 
In this benchmark result, the classification performance is measured by the average testing accuracy obtained with $k$-fold cross-validation, in our experiment, we set $k=5$. 
It's known there exists a valid upper threshold for  $\lambda_{max}$ in logistic regression~\cite{koh2007interior}, when the regularization parameter is larger than that, the cardinality of the solution will be zero. Therefore $\lambda$ is usually selected as a fraction proportion of $\lambda_{max}$. We choose a 10-length path for $\lambda$, where the fraction is $0.01, 0.02, 0.05, 0.07, 0.1, 0.2, 0.3, 0.5, 0.7, 0.8$ respectively. All the other parameters are set  as suggested in the original papers.
In Table~\ref{tab:1} and Table~\ref{tab:2}, we show the testing accuracy with various $\lambda$ for $\ell_1$-norm and SCAD regularizer item, respectively. 


We compare the efficiency of our proposed methods with other methods. 
We show the computation time for the training with 1000 and 500 features. 
The number of samples is fixed at 1000 and with $\lambda / \lambda_{max}$ = 0.1. 
From Figure~\ref{fig:efficiency} we can see that in the $\ell_1$-norm regularized logistic regression problem, our proposed methods (ISTA-BB, ISTA-reverse, and FISTA-Lip) require less computation time to converge. The proximal Newton method is not included in the figure because its running time is way higher than the others, making it hard to visualize the time in the same figure. 
The nonconvex regularized logistic regression  follows a similar path, ISTA-BB and ISTA-reverse have better performance than GPGN and HONOR in terms of less time consumption. 
\begin{figure*}[ht!]
	\centering 
\subfigure[$\ell_1$-norm, feature]{\label{fig:21}\includegraphics[width=43mm]{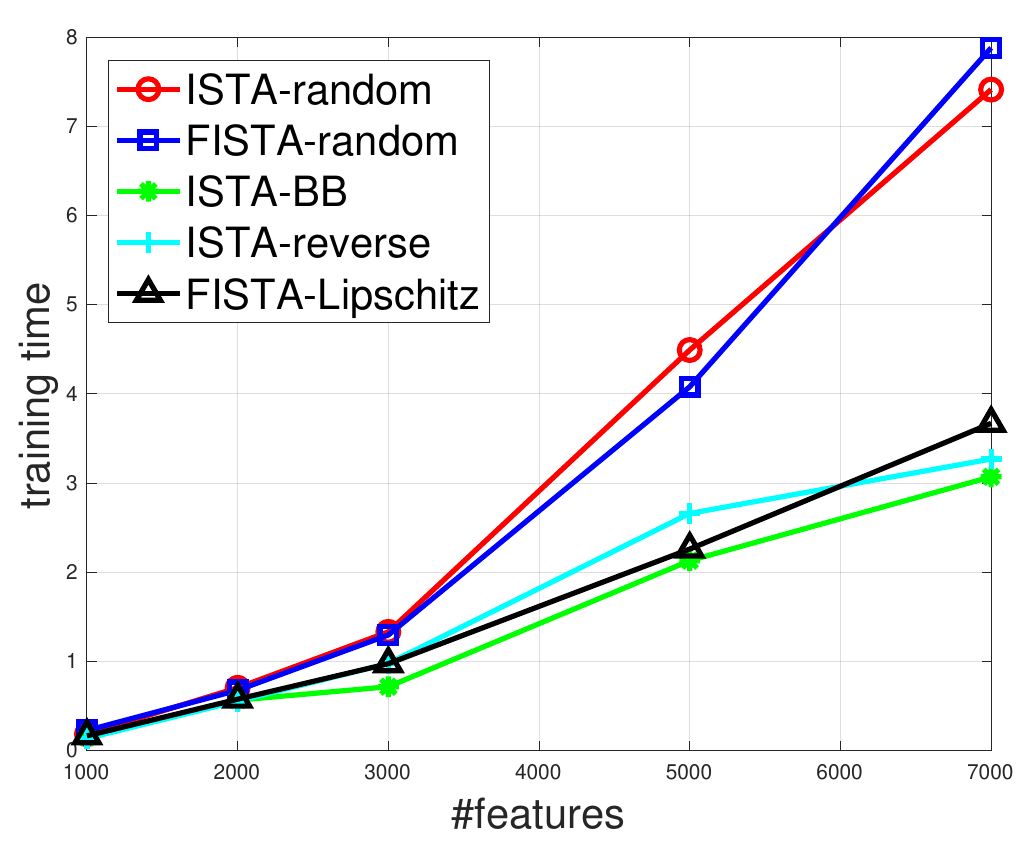}}
\subfigure[$\ell_1$-norm, sample]{\label{fig:22}\includegraphics[width=43mm]{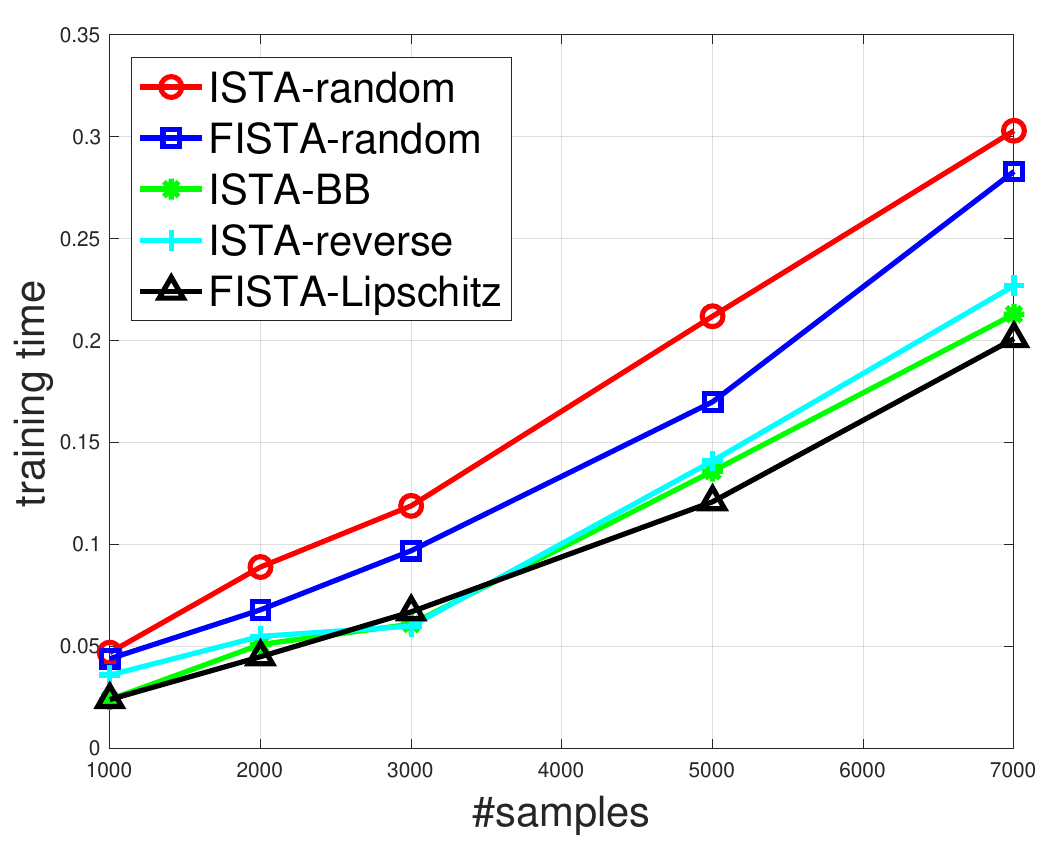}}
\subfigure[SCAD, feature]{\label{fig:23}\includegraphics[width=43mm]{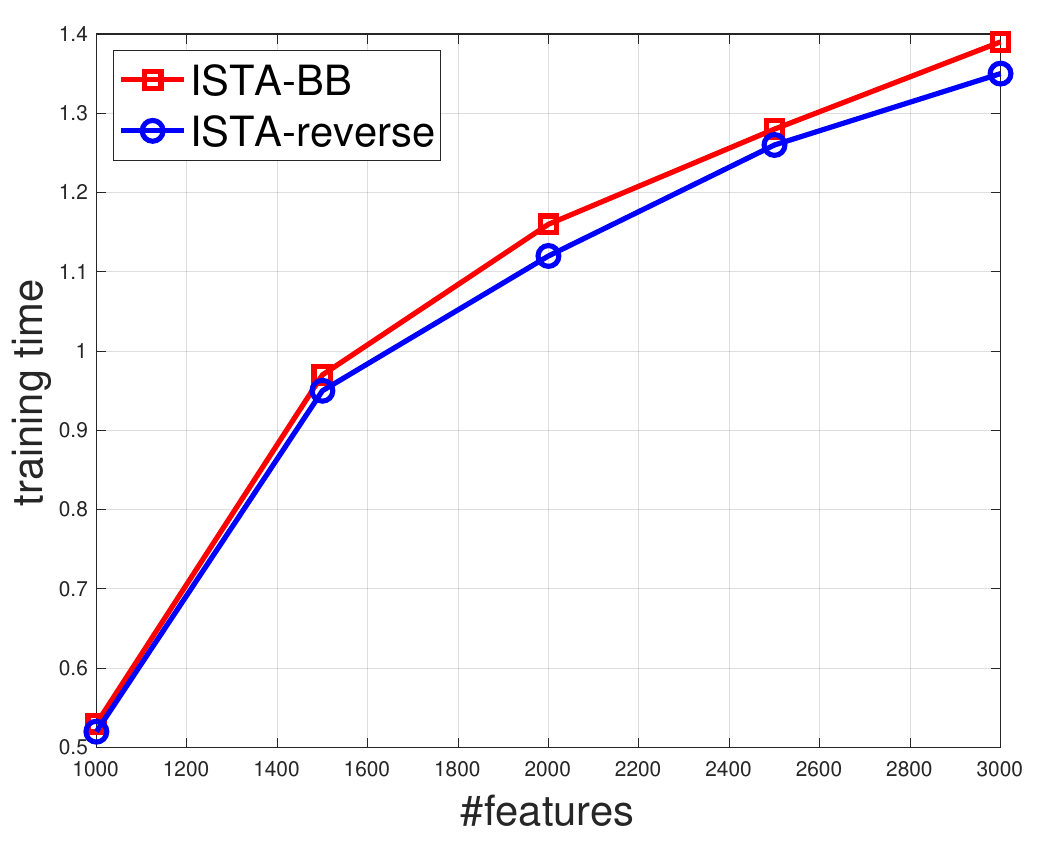}}
\subfigure[SCAD, sample]{\label{fig:24}\includegraphics[width=43mm]{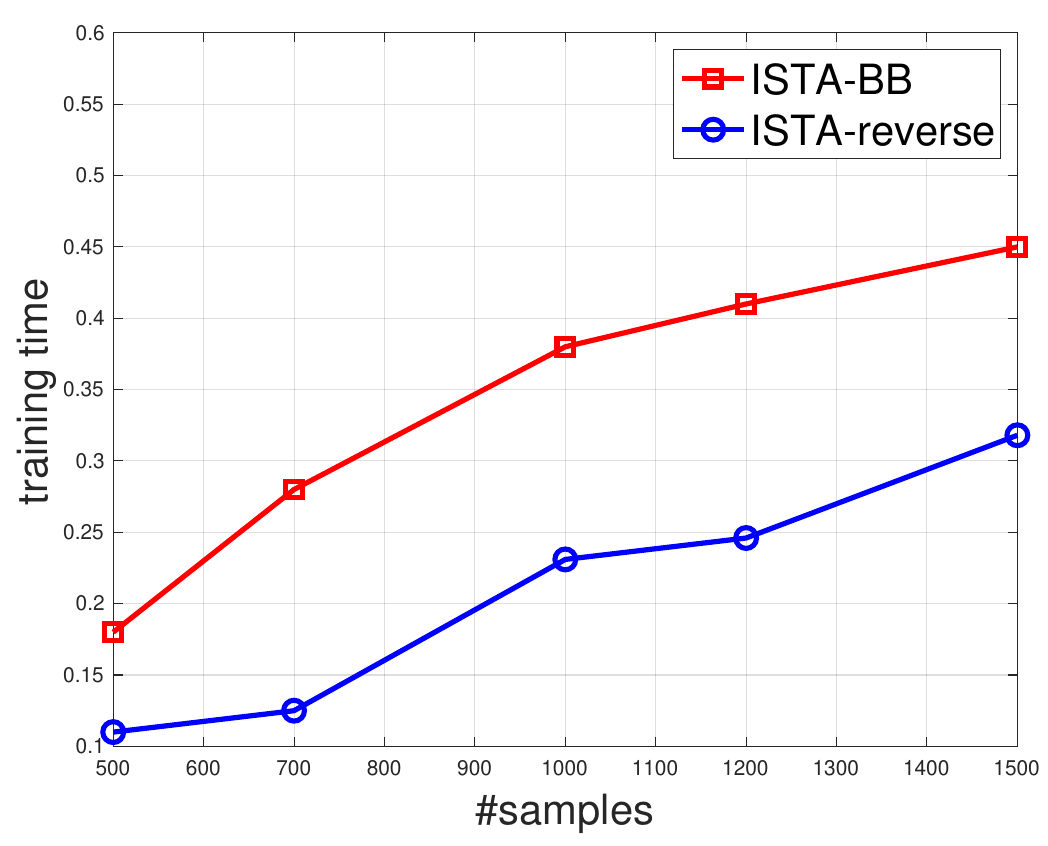}}
	\caption{Scalability study -- the influence of number of features and samples on  training time for each method to converge. From left to right: $\ell_1$-norm regularized with varying features, $\ell_1$-norm regularized with varying samples, SCAD regularized  with varying features and SCAD regularized with varying samples.}
	\label{fig:scale}
\end{figure*}
We also conduct numerical experiments to show the scalability of our methods. We varied the number of features and the number of samples to show how our algorithms perform with the size of samples increase. The results are illustrated in Figure~\ref{fig:scale}. 
For our proposed methods, the computation time is quite low even with a large number of features and samples, the trends of increasing computation time with increased features/samples are quite stable, there is no explosive increase in the computation time with the rapid growth in features/samples. 



\begin{figure}[!h]
\begin{minipage}[t]{0.485\linewidth}
    \centering
    \includegraphics[width=0.95\textwidth]{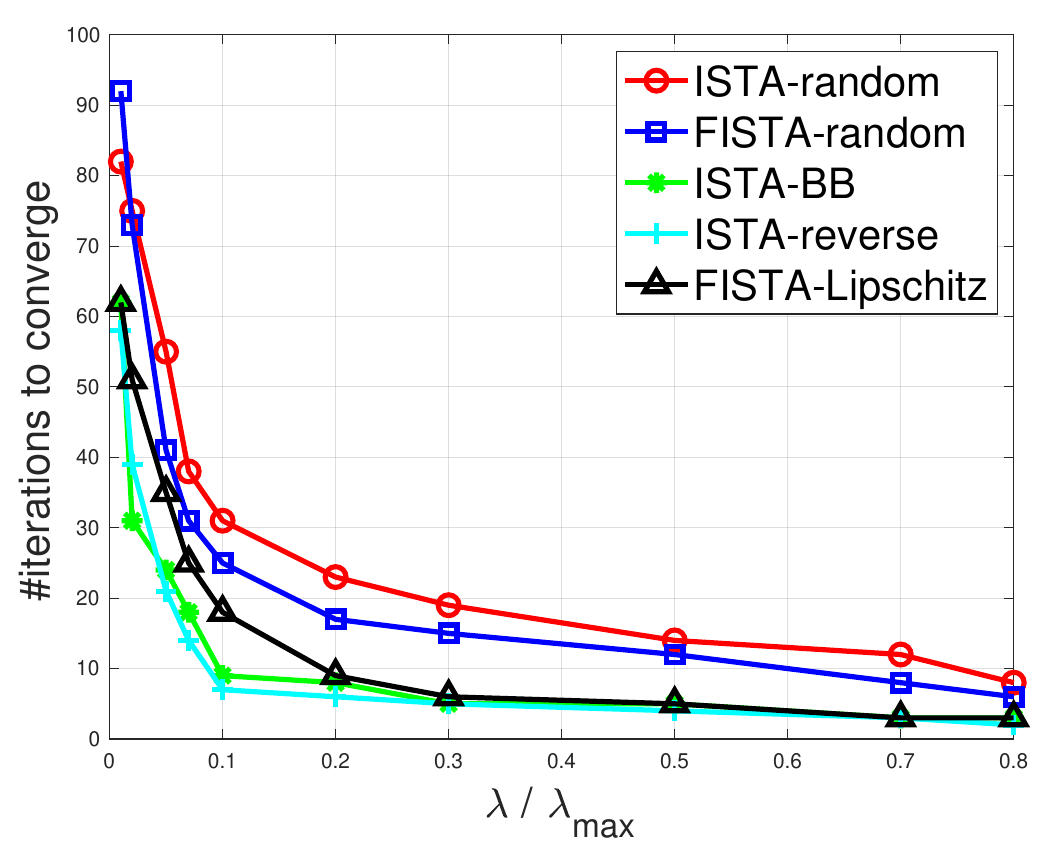}
    \caption{$\mathcal{L}_1$-norm regularized: comparison of iterations for each method to converge with varying $\lambda$.}
    \label{fig:convexLambdaIter}
\end{minipage}
\hspace{0.1cm}
\begin{minipage}[t]{0.485\linewidth} 
    \centering
    \includegraphics[width=0.95\textwidth]{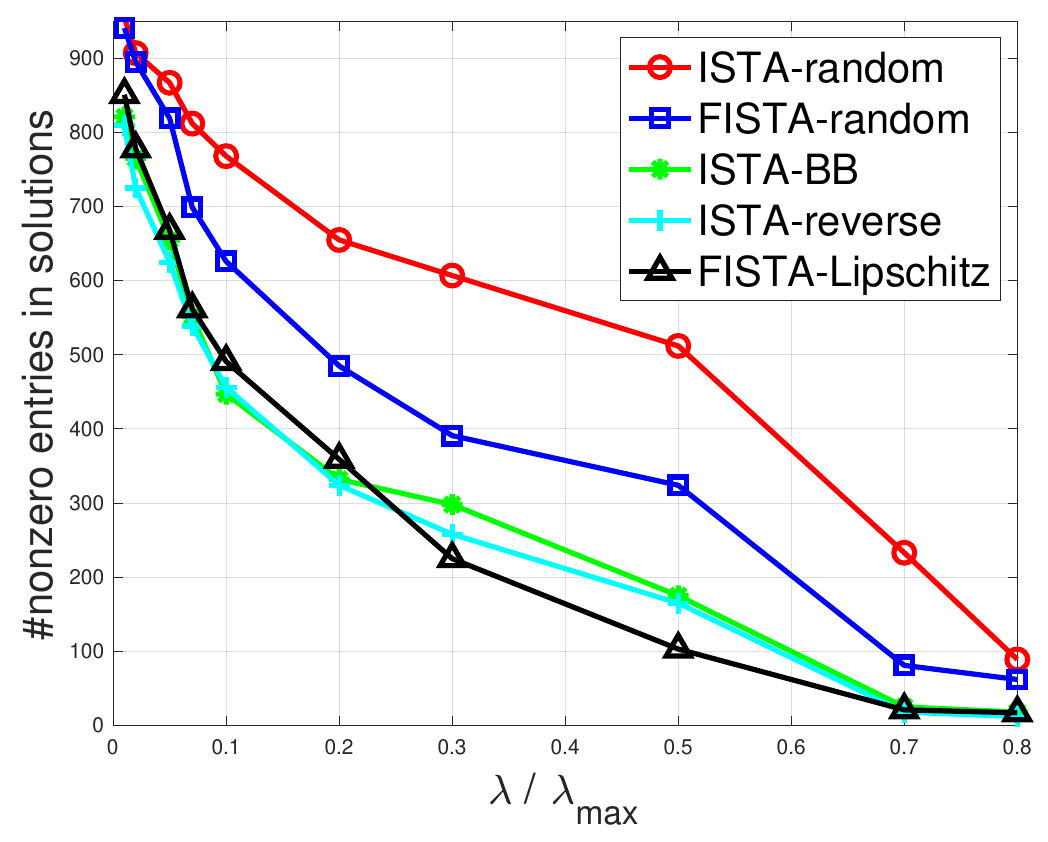}
    \caption{$\mathcal{L}_1$-norm regularized: comparison of sparsity of solutions obtained by each method with varying $\lambda$.}
    \label{fig:convexLambdaSparsity}
\end{minipage}        
\end{figure}  




We present the convergence  of our proposed methods along the regularization path, using a \textit{Ionosphere} data from the UCI repository, which is shown in Figure~\ref{fig:convexLambdaIter}. Similarly, we also include the original ISTA and FISTA methods to show the superiority of our proposed. The figure shows that our methods have similar performance in terms of convergence, but require less  iterations  to converge han that of vanilla ISTA and FISTA. 
We also explore the sparsity of the solutions, which is shown in Figure~\ref{fig:convexLambdaSparsity}. The numbers of nonzero entries obtained by our methods are pretty similar to each other, and the solutions are more sparse than the counterparts. In general, we can see that the regularization parameter $\lambda$ has effects on the number of iterations to converge and the sparsity of optimal solutions. 
Also, we note that the proximal Newton method is not able to induce sparsity even with large $\lambda$.


\section{Conclusions} \label{sec:5}

In this paper, we propose new optimization frameworks to solve sparse logistic regression problem which work for both convex and nonconvex regularization terms.
Experimental results on benchmark datasets with both types of regularizers demonstrate the advantages of our proposed algorithms compared to others in terms of both accuracy and efficiency. 

\bibliographystyle{plain}
\bibliography{sample-base}

\end{document}